\documentclass[journal]{IEEEtran}
\usepackage{graphicx,subfigure}
\usepackage{mathtools, cuted}
\usepackage{cite}
\usepackage{picinpar}
\usepackage{amsmath}

\usepackage{amssymb}
\usepackage{url}
\usepackage{flushend}
\usepackage[latin1]{inputenc}
\usepackage{colortbl}
\usepackage{soul}
\usepackage{multirow}
\usepackage{color}
\usepackage{alltt}
\usepackage{siunitx}
\usepackage{epstopdf}
\usepackage{pbox}
\usepackage{gensymb}
\usepackage{amsthm}
\usepackage{cuted}
\usepackage{multicol,lipsum}
\usepackage{algorithm}
\usepackage[noend]{algpseudocode}

\usepackage{hyperref}

\newtheorem{theorem}{Theorem}

\theoremstyle{definition}
\newtheorem{definition}{Definition}

\begin{document}
\title{Interpretable Fault Diagnosis of Rolling Element Bearings with  Temporal Logic Neural Network}
\author{
	\vskip 1em
	Gang Chen,  
	Yu Lu, Rong Su, \emph{Senior Member, IEEE} and Zhaodan Kong, \emph{Member, IEEE}
	\thanks{
		Manuscript received Month xx, 2xxx; revised Month xx, xxxx; accepted Month x, xxxx.
		
		The support from Singapore National Research Foundation Delta-NTU Corporate Lab Program (SMA-RP14) and from Singapore Ministry of Education AcRF Tier 1 2018-T1-001-245 (RG 91/18) are gratefully acknowledged.
		
		Gang Chen, Yu Lu and Rong Su are  with school of electrical and electronic engineering, Nanyang Technological University, Singapore 639798 (e-mail:gang.chen@ntu.edu.sg,   rsu@ntu.edu.sg).
		
		Zhaodan Kong is with the  school of mechanical and aerospace engineering, University of California, Davis, Davis, CA, 95616 (zdkong@ucdavis.edu).
 
	}
}

\maketitle
	
\begin{abstract}
Machine learning-based methods have achieved successful applications in machinery fault diagnosis. However, the main limitation that exists for these methods is that they operate as a black box and are generally not interpretable. This paper proposes a novel neural network structure, called temporal logic neural network (TLNN), in which the neurons of the network are logic propositions.  More importantly, the network can be described and interpreted as a weighted signal temporal logic. TLNN not only keeps the nice properties of traditional neuron networks but also provides a formal interpretation of itself with formal language. Experiments with real datasets show the proposed neural network can obtain highly accurate fault diagnosis results with good computation efficiency. Additionally, the embedded formal language of the neuron network can provide explanations about the decision process, thus achieve interpretable fault diagnosis.
\end{abstract}


\markboth{}%
{}

\section{Introduction}
In the past few decades, there is a trend that modern equipments   are becoming more intelligent, sophisticated, and integrated. The increasing complexity of these equipments demands fault diagnosis of these crucial components to guarantee their safety and reliability \cite{haidong2018intelligent}. For instance, rolling element bearings are widely used components in modern rotary machines. Their faults may lead to the fatal breakdown of modern machines. Therefore, intelligent fault diagnosis for rolling element bearings is urgently needed to prevent catastrophic failures, enhance safety and reliability, and reduce maintenance costs.

Nowadays, with the increasing volume of industrial data and the rapid development of artificial intelligence (AI) and machine learning (ML) techniques, fault diagnosis for bearings with machine learning algorithms is receiving more and more attention \cite{yang2020interpreting}. Compared with traditional fault diagnosis methods, which need human experts to manually extract features or interpret the collected data, fault diagnosis with machine learning algorithms are data-driven. These methods use ML algorithms to learn the features automatically from the collected data, which replaces the tedious feature extraction procedure by experts. Especially in recent years, deep learning (DL) methods have achieved great success in pattern recognition \cite{goodfellow2016deep}, and have attracted intensive attention in the field of fault diagnosis for bearings\cite{he2017deep}. Many kinds of neural network structures have been proposed and applied to fault diagnosis, such as deep belief network (DBN) \cite{shao2017electric}, convolutional neural network (CNN) \cite{zhao2019deep}, long-short-term memory network (LSTM) \cite{hao2020multisensor}, and autoencoder (AE) \cite{chen2017multisensor}. Tested on some open-source  fault diagnosis datasets, DL-based fault diagnosis methods can usually  achieve good results and outperform other methods. 

Although the DL-based methods have achieved  many successes in fault diagnosis, there are still some remaining issues needed to be tackled associated with them. Usually, these methods consist of three main steps: 1) sensor signal acquisition; 2) feature extraction and selection;  3) fault classification. The DL neural networks are used for steps 2) and 3). The main limitation of these methods in fault diagnosis is that they operate as a black box and are not interpretable, i.e., they do not provide the explanations about \textit{how and why} they make a decision and reveal the fault mechanism.  However, in practice, such interpretations are important since they are useful not only for the trustworthiness of the decision itself but also for further active maintenance.
  
Many scholars have tried to propose interpretable neural networks to equip the fault diagnosis results with interpretability \cite{zhao2021interpretable, abid2019robust}. Interpretability means understanding and explaining neural network models. However, to the best of our knowledge, there is neither complete theory and method of pure mathematical analysis nor formal mathematical definition for interpretability, at least in the context of DL-based fault diagnosis \cite{yang2020interpreting}. To devise an interpretable neural network for intelligent fault diagnosis, most existing works use visualization to connect   key explainable features to the decision  process and provide a visual and empirical explanation for users. The goal of the interpretation is to obtain trust from the users, by making the users   feel the decision made by the network is in line with their understanding of the physical process. For example, Firas Ben Abid \textit{et al.} \cite{abid2019robust}   proposed a new DL architecture called deep-SincNet, which utilizes a traditional end-to-end scheme as usual, but it can automatically learn  interpretable fault features from the raw signals and accordingly finalizes the fault diagnosis process.  Ahmed Ragab \textit{et al.} \cite{ragab2018fault}  applied the Logical Analysis of Data (LAD) for fault diagnosis in industrial chemical processes, in which a machine learning  technique is used to discover   interpretable patterns, which can be linked to underlying physical phenomena and provide interpretation of the decision process.  John Grezmak  \textit{et al.} \cite{grezmak2019interpretable} used the  time-frequency spectra images to train a CNN, and then applied Layer-wise Relevance Propagation (LRP) to    build the connection about which values in the input signal contribute the most to the diagnosis results, thereby providing an improved understanding of how the CNN makes decision about fault classification.

Visualization-based DL can provide initial interpretations, but the interpretations are empirical and not formal, where the users need many numerical experiments  and background knowledge to check the relationship between visualization and physical process and  approach an interpretation of the results. Moreover, such interpretations are mostly based on only one layer feature of the DL, which cannot provide an explanation about how the decision is made formally. To address these issue, interpretable fault diagnosis based on formal languages has attracted much attention in the field of monitoring \cite{chen2020frequency,chen2020formal}. Formal language-based fault diagnosis methods try to use a neural network to learn a formal language, e.g., temporal logic \cite{deshmukh2017robust} and spectral temporal logic \cite{chen2020frequency}, to describe the fault behaviours among the signals. The learned formal languages can classify the faults,  be understood by users, and give an explanation about the results, where the faulty and normal signals can be easily described by a formal but human readable sentences.

Fault diagnosis with formal language has been widely used in monitoring tasks \cite{chen2020formal, dokhanchi2014line} in the past decade, which has achieved great success. However, there are many challenges to apply  formal language to interpretable fault diagnosis tasks for existing methodologies. Firstly, many methods assume the structure of formal languages are given or assume the formal languages are constructed with a set of given atomic words, usually defined by experts, but with unknown parameters. As such neural network is used to learn the optimal parameters for the languages, which is not applicable and too conservative for modern complex equipment. Secondly,   learning a formal language with neural network only focuses on fault diagnosis performance with testing dataset.  Even though the formal language is interpretable,   the learning process itself is also a black-box process, which is not fully interpretable as the visualization-based methods, thus cannot provide a formal guarantee for the performance.

To address the aforementioned issues, this paper proposes a new framework to infer formal languages for fault diagnosis of rolling element bearings. In the proposed framework, we do not learn the formal language directly, but the language is embedded in the neural network and encoded by the parameters of the neural network. The proposed neural network is called temporal logic neural network (TLNN). After the network is trained with training data, we can formally map the parameters of the network to a temporal logic formula.  Moreover, the activation function of the network is related to the semantics of a formal language, which provides an explanation for the overall network with the formal language.  Compared with state-of-the-art methods, our contributions are as follows.
\begin{enumerate}
\item We propose a new neural network and its parameters learning method, called TLNN, whose structure can be mapped to a formal language, thus it is interpretable with formal languages.  
\item We apply the proposed neural network to fault diagnosis with real data sets, and the results show the proposed neural network can learn temporal logic formulas for fault diagnosis and the network itself can be interpreted with the formulas.

\end{enumerate}

The layout of this paper is as follows: Section \ref{sub:prelim}   introduces  the preliminaries of this paper and defines the atom temporal logic. Section \ref{sub:TLNN} defines the structure of the TLNN. Section \ref{sec:learning} introduces the learning process and gradients for each parameters in TLNN. Section \ref{sub:case} demonstrates the proposed method with real data sets and compares the proposed method with state-of-the-art techniques, and Section \ref{sub:conclusion} gives the conclusions.

\section{Preliminaries}
\label{sub:prelim}
\subsection{Weighted Signal Temporal Logic}

Denote $A$ and $B$ to be two sets, and  $\mathcal{F}(A,B)$ to be the set of all functions from $A$ to $B$. Denote a time domain to be $\mathbb{D}:=\{k\tau_{0}|k\in Z_{\geq 0}\}$, then a \emph{discrete-time, continuous-valued signal} is a function $x \in \mathcal{F}(\mathbb{D},\mathbb{R}^n)$. For example,   $x(t)$   denotes the value of signal $x$ at time $t$. Moreover, we use the terms ``signal'' and ``time series'' interchangeably. 

\begin{definition}
Signal temporal logic (STL) is a temporal logic defined over signals \cite{chen2020temporal}. Its syntax is defined recursively as: 
\begin{equation}
\varphi ::= \mu| \varphi_1 \wedge \varphi_2 | \varphi_1 \vee \varphi_2 | \lozenge_{[\tau_{1},\tau_{2})} \varphi | \square_{[\tau_{1},\tau_{2})} \varphi,
\label{eq:STL}
\end{equation}
where $\tau_{1}\leq \tau_{2}$ and $\tau_{1},\tau_{2}\in \mathbb{Z}_{\geq 0}$ are indexes for time, and $\mu$ is a predicate over a signal,  defined as $f(x(t))\sim c$ with $f \in \mathcal{F}(\mathbb{R}^n,\mathbb{R})$ being a function, $\sim \in \{  <, \geq\}$, and $c \in \mathbb{R}$ being a constant. The Boolean operators $\vee$ and $\wedge$ are disjunction (``or'') and conjunction (``and''), respectively. The temporal operators before STL formula $\varphi$, $\lozenge_{[\tau_{1},\tau_{2}]} \varphi$ and $\square_{[\tau_{1},\tau_{2}]} \varphi$, stand for ``eventually'' (to be true at least once) formula $\varphi$ being true between time interval $[\tau_{1},\tau_{2}]$ and ``always'' (to be true all the time) formula $\varphi$ being true between time interval $[\tau_{1},\tau_{2}]$, respectively.  
\end{definition}
In practice, using temporal logic to describe the behaviors of a system involves using a set of sub-specifications with different importance or priories. The expressiveness of traditional STL with the quantitative semantic defined in \cite{chen2020temporal} does not allow for specifying such priorities, since all the sub-specifications are equally important \cite{mehdipour2020specifying}. Moreover, the satisfaction status of traditional STL depends on critical points of the signal under investigation, caused by min and max operators, which is sensitive to noise and not differentiable, thus restricting the application of STL to dynamical systems in noisy environments. To address these issues, assigning different weights to different sub-specifications has been proposed in \cite{mehdipour2020specifying,varnai2020robustness,gilpin2020smooth,mehdipour2019arithmetic}.
These extensions of STL can be generalized as weighted STL (wSTL), the syntax of which can be defined as follows.
\begin{definition}
(\textit{wSTL Syntax} \cite{mehdipour2020specifying} ): The syntax of wSTL is modified from that of the traditional STL as follows:
\begin{equation}
\varphi:=\mu |\neg\varphi|\wedge_{i=1:N}^{w}\varphi_{i}|\vee_{i=1:N}^{w}\varphi_{i}|\lozenge_{[\tau_{1},\tau_{2})}^{w} \varphi | \square_{[\tau_{1},\tau_{2})}^{w} \varphi,
\label{eqn:syntax}
\end{equation}
where the predicate $\mu$ and all the Boolean and temporal operators have the same semantics as in STL, except that there exists a weight $w_{i}$  assigned to each sub-formula $\varphi_{i}$ and $w=[w_{i}]_{i=1:N}\in \mathbb{R}^{N}_{\geq 0}$
\end{definition}
wSTL is  also equipped with quantitative semantics as STL, which  are defined as follows.
\begin{definition}
Given a wSTL formula $\varphi$ and  a signal $x$, the weighted robustness degree $\rho^{w}(x,\varphi, t)$ at time $t$ is recursively defined as follows:
\begin{equation*}
\begin{array}{lll}
\rho^{w}(x, f(x)\geq c,  t) &:= \frac{w}{2} (f(x(t))-c),\\
\rho^{w}(x, f(x)< c,  t) &:= \frac{w}{2} (c - f(x(t))),\\
\rho^{w}(x, \neg\varphi, t)&: = -\rho^{w}(x, \varphi, t),\\
\rho^{w}(x, \wedge_{i=1:N}^{w}\varphi_{i}, t)&: = g^{\wedge}(w,[\rho^{w}(x,\varphi_{i},t)]_{i=1:N}),\\

\rho^{w}(x, \vee_{i=1:N}^{w}\varphi_{i},t)&:=g^{\vee}(w,[\rho^{w}(x,\varphi_{i},t)]_{i=1:N}),\\
\rho^{w}(x, \square_{[\tau_{1},\tau_{2}]}^{w}\varphi, t)&:=g^{\square}(w, [\rho^{w}(x,\varphi,t)]_{t'\in [t+\tau_{1},t+\tau_{2}]}),\\
\rho^{w}(x, \lozenge_{[\tau_{1},\tau_{2}]}^{w}\varphi, t)&:=g^{\lozenge} (w,[\rho^{w}(x,\varphi,t)]_{t'\in [t+\tau_{1},t+\tau_{2}]}),\\
\end{array}
\label{eqn:weightedrobust}
\end{equation*}
where $g^{\wedge}:\mathbb{R}_{\geq 0}\times \mathbb{R}^{N}\rightarrow \mathbb{R}$, $g^{\vee}:\mathbb{R}_{\geq 0}^{N}\times\mathbb{R}^{N}\rightarrow \mathbb{R}$, $g^{\square}:\mathbb{R}^{\tau_{2}-\tau_{1}}_{\geq 0}\times \mathbb{R}^{\tau_{2}-\tau_{1}}\rightarrow \mathbb{R}$, and  $g^{\lozenge}:\mathbb{R}^{\tau_{2}-\tau_{1}}_{\geq 0}\times \mathbb{R}^{\tau_{2}-\tau_{1}}\rightarrow \mathbb{R}$ are activation functions associated with $\wedge, \vee, \square$ and $\lozenge$ operators, respectively. The activation functions $\Upsilon =(g^{\wedge},g^{\vee},g^{\square},g^{\lozenge})$ define the robustness degree of wSTL and determine the properties of wSTL, e.g., soundness as defined in   \cite{mehdipour2020specifying}.
\end{definition}

\section{Temporal Logic Neural Network}
\label{sub:TLNN}
This section introduces the structure of a multiple-input single-output temporal logic neural network (TLNN). Fig. \ref{fig:TLNN} shows the proposed five-layer TLNN structure. The proposed TLNN maps a signal  $x = x(0),\cdots,x(t),\cdots,x(n)$ to a weighted robustness degree with the wSTL formula embedded in the neural network. The detailed mathematical functions of each layer are introduced hereinafter.

\begin{figure}[hbtp]
\centering
\includegraphics[scale=0.35]{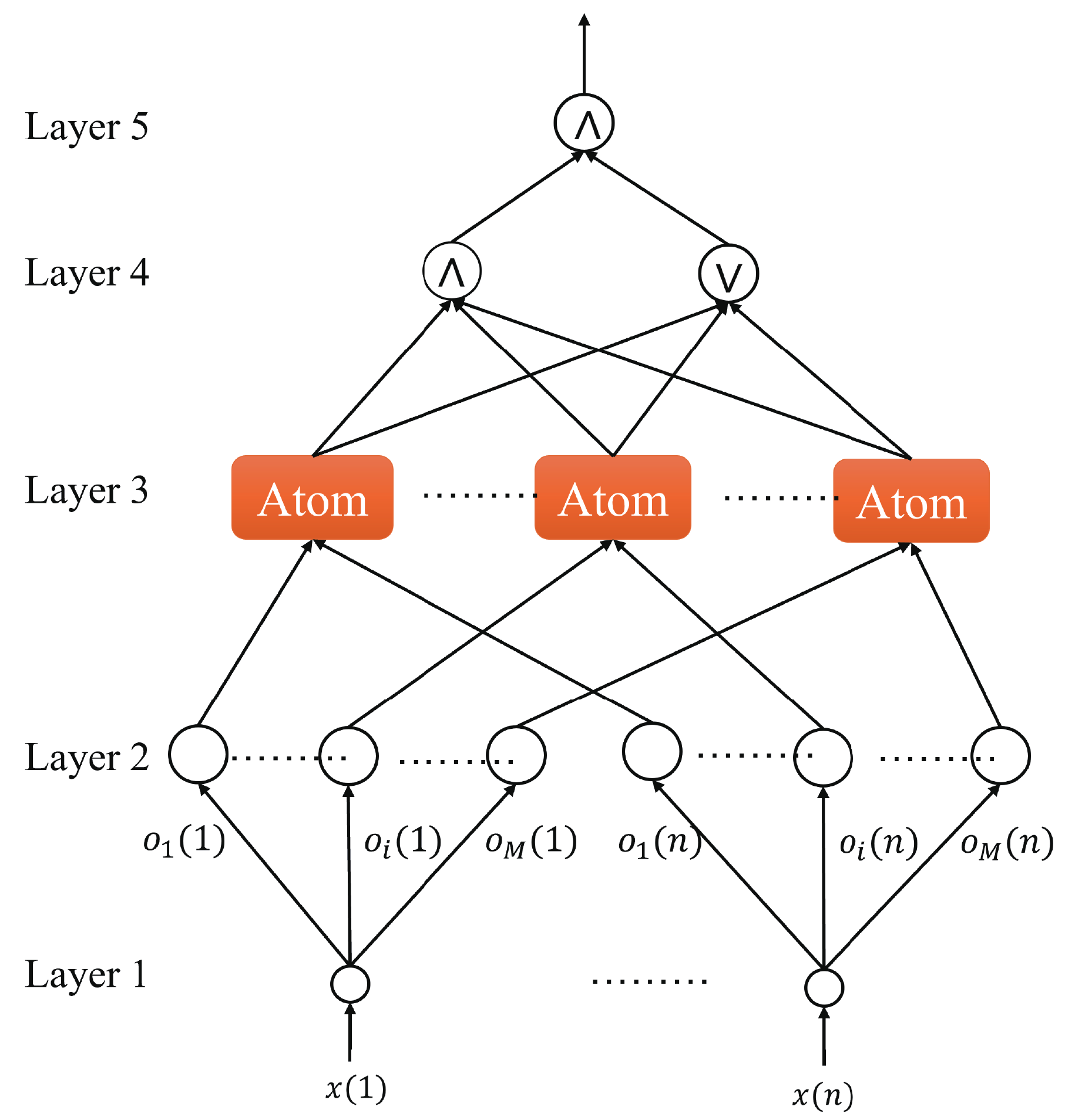}
\caption{Proposed five-layer temporal logic neural network (TLNN) structure, where each computation rule in Layer 3 formulates a set of  wSTL sub-formula and these sub-formulas in Layer 3  are sequentially combined with ``and'' or ``or'' operators in layer 4 and layer 5.}
\label{fig:TLNN}
\end{figure}

\textbf{Layer 1} (Input layer): This layer contains $n$ neurons, where $n$ is the length of the signal. Each neuron  has $M$ output values, where $M$ is the number of predicates encoded by the neural network. The ith output value for  the  neuron $t$, denoted as $o_{i}(t)$ and defined as:
\begin{equation}
o_{i}(t) = x(t), i=1,2,\cdots, M, t= 0,1,\cdots, n
\end{equation}
where the input signal is $x =x(0),x(1),\cdots,x(n)$.

\textbf{Layer 2} (Predicate layer): Each neuron in this layer is a mapping function that generates a predicate  in (\ref{eqn:syntax}), which is defined as:
\begin{equation}
\mu_{i}(t)=f (x(t))-W_{i}^{1}(t)
\end{equation}
where $W_{i}^{1}(t)$ is the weight  assigned to the $i-th$ output of   neuron $t$. In this paper, based on the syntax and semantics of  wSTL, the weight vector $W^1=[W_{1}^{1},W_{2}^{1},\cdots,W_{M}^{1}]$ for each neuron to be the same, i.e., $W_{i}^{1}(1)=W_{i}^{1}(2),\cdots,=W_{i}^{1}(n)$.  

\textbf{Layer 3} (Atomic formula layer): This layer defines $M$ atomic formulas (sub-formulas) that construct  the final wSTL formula. This layer consists of four type of neurons, i.e., ``always'' neuron, ``eventually'', ``eventually always'' neuron, and ``always eventually'', respectively. Each neuron defines a formula with the form $\square_{[\tau_{1},\tau_{2}]}^{w}\varphi_{i}$, $\lozenge_{[\tau_{1},\tau_{2}]}^{w}\varphi_{i}$,  $\square_{[0, \tau_{0}]}\lozenge_{[\tau_{1},\tau_{2}]}^{w}\varphi_{i}$ or $\lozenge_{[0,\tau_{0}]}\square_{[\tau_{1},\tau_{2}]}^{w}\varphi_{i}$. The inside of the neuron is shown in Fig.\ref{fig:atomic}, which includes an autoencoder and an activation function $g(f,W)$. The activation  functions are chosen from $(g^{\wedge},g^{\vee},g^{\square},g^{\lozenge})$ as defined in the weighted robustness degree definition and their combinations. The encoder maps the inputs from layer 2 to a temporal interval $[\tau_{1},\tau_{2}]$, and the decoder maps the interval to a matrix $W$ that will be used by the activation function, which will be described in detail next.

\textbf{(1)} \underline{$\square_{[\tau_{1},\tau_{2}]}^{w}\varphi_{i}$}: In this case, the encoder is a two-layer feed-forward neuron network followed by a quantization operator. Denote the input to the encoder for the $i-th$ neuron in layer 3 as $\rho^{2}_{i}\in\mathcal{R}^{1\times n}$, then the output of this neuron is 
\begin{equation}
\begin{bmatrix}
h_{1}\\
h_{2}
\end{bmatrix}
=E(\theta_{e}, \rho^{2}_{i})
\end{equation}
where $\theta_{e}$ is the parameter  of the encoder, $h_{1}, h_{2}\in\mathbb{R}$ are the two values that encode  the time interval and $\rho^{2}_{i}$ is the robustness from Layer 2.  Since we assume  discrete time values, the time interval should be quantized into integer values with the quantization operator $Q_{Z}(h)$ defined in \cite{gong2019differentiable}. Given the bit width $b$ and the real value $h$, which is in the range $(l, u)$, the complete quantization  process can be defined as:
\begin{equation}
Q_{Z}(h) = round(\frac{h}{\bigtriangleup})\bigtriangleup
\label{eqn:quantization}
\end{equation}
where the original range $(l, u)$ is divided into  $2^{b}-1$ intervals $\mathcal{P}_{i}, i\in [0,1,\cdots,2^{b}-1]$, and $\bigtriangleup = \frac{u-l}{2^{b}-1}$.   The derivative of the quantization function is zero almost everywhere, which makes the training process unstable and decreases the learning accuracy. To address this issue, the soft quantization is used during training process as follows.

\begin{figure}[hbtp]
\centering
\includegraphics[scale=0.5]{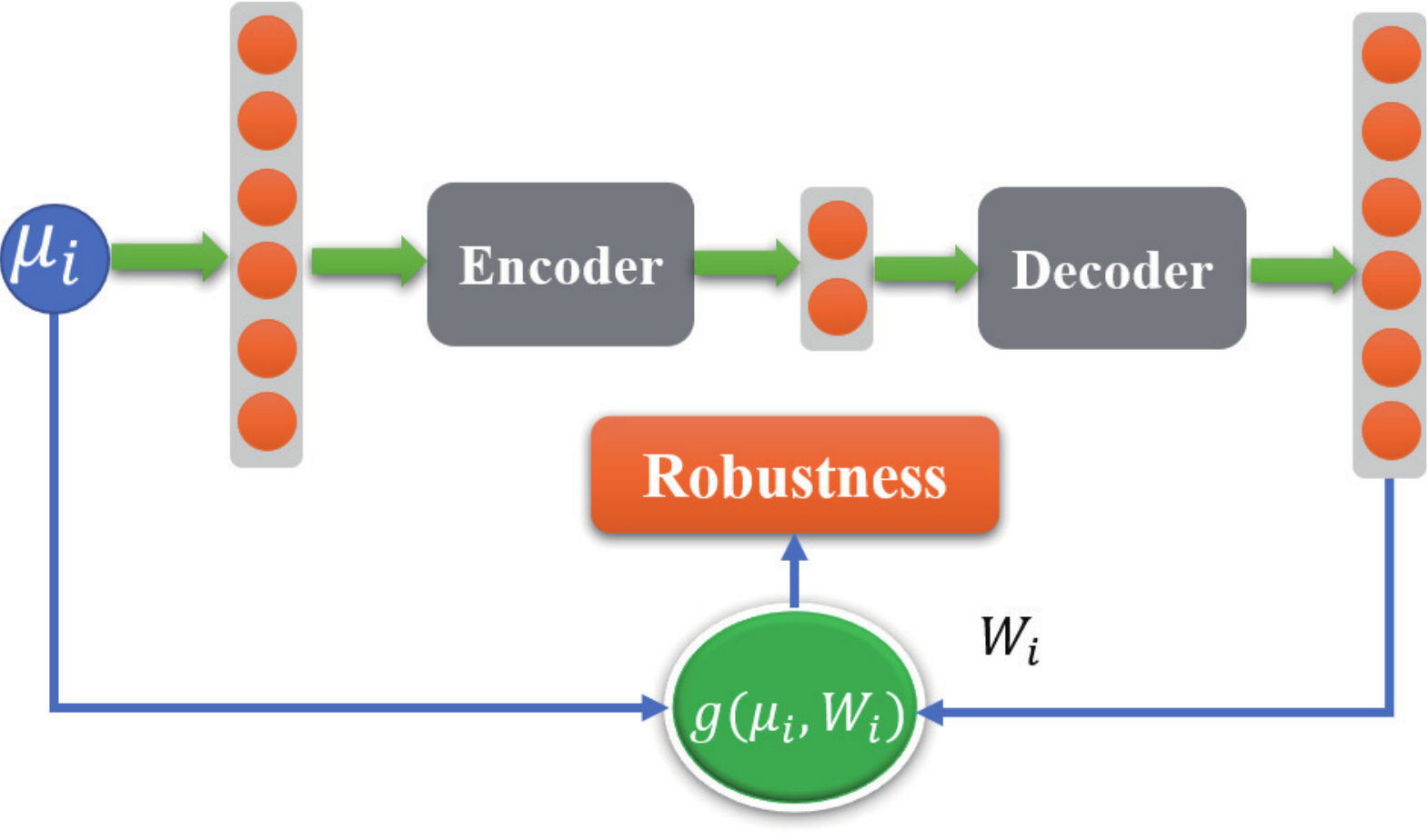}
\caption{Atomic formula construction framwork in Layer 3 of the TLNN.}
\label{fig:atomic}
\end{figure}

\begin{equation}
Q(x) =
\left\lbrace 
\begin{array}{lll}
l, x<l\\
u, x>u\\
l+\bigtriangleup (i+\frac{\kappa (x)+1}{2}), x\in\mathcal{P}_{i}\\
\end{array}\right. 
\end{equation}
where $l$ and $u$ are the lower and upper bounds of the quantization. $\mathcal{P}_{i}$ is a   quantitative interval and $\kappa (h)$ is a differentiable asymptotic function that has good approximation performance, defined as 
\begin{equation}
\kappa (x) = \frac{1}{tanh(0.5k\bigtriangleup)} tanh(k(x-l-(i+0.5)\bigtriangleup)),
\label{eqn:quantizationsoft}
\end{equation}
where $k$ is a coefficient that determines the shape
of the asymptotic function. Then the time interval is defined as  $\tau_{1} = Q_{Z}(h_{1})$ and $\tau_{2} = Q_{Z}(h_{1})+Q_{Z}(h_{2})$, where the lower and upper bounds are set to 0 and $n$, respectively. 

The decoder is also a two-layer feed-forward neural network followed with a mapping. The output of the decoder is a vector $W^{2} \in \mathbb{R}^{n \times 1}$, and the $j-th$ element of  is defined as 
\begin{equation}
W^{2}_{j} = D(\theta_{d},[\tau_{1},\tau_{2}])_{j},
\label{enq:indxmap}
\end{equation}
where $D(\theta_{d},[\tau_{1},\tau_{2}])$ is the output of the feed-forward neural network and $\theta_{d}$ is the parameters of the decoder network. 

With the output of layer 2 and the output of the autoencoder, the activation function can be defined as:
\begin{equation}
g^{\square}(\rho^2,W^{2}) = \rho^{w}(x, \square_{[\tau_{1},\tau_{2}]}^{w^{2}_{i}}\rho^{2}_{i}, t)
\end{equation}
where $\rho^{w}(x, \square_{[0,n]}^{w^{2}_{i}}\rho^{2}_{i}, t)$ is defined in  (\ref{eqn:weightedrobustness}), which is inspired by the arithmetric-geometric integral mean robustness in \cite{mehdipour2019arithmetic}. In  (\ref{eqn:weightedrobustness}), $[q]_{+} =\left\lbrace \begin{array}{ll}
q, q>0\\
0 \text{ otherwise}
\end{array}\right. $ and $[q]_{-} = -[-q]_{+}$, where $q:\mathbb{R}^{n}\rightarrow \mathbb{R}$ is a real function and $q =[q]_{+}+[q]_{-1}$.

\textbf{(2)} \underline{$\lozenge_{[\tau_{1},\tau_{2}]}^{w}\mu_{i}$}: In this case, the structure of the autoencoder is the same as the case in  \underline{$\square_{[\tau_{1},\tau_{2}]}^{w}\varphi$}, but the activation function is different and defined as 
\begin{equation}
g^{\lozenge}(\rho^2_{i}, W^{2}) = \rho^{w}(x, \lozenge_{[\tau_{1},\tau_{2}]}^{W^{2}}\rho^{2}_{i}, t)
\end{equation}
where $\rho^{w}(x, \lozenge_{[\tau_{1},\tau{2}]}^{W^{2}}\rho^{2}_{i}, t)$ is defined in  (\ref{eqn:weightedrobustness}).

\textbf{(3)} \underline{$\square_{[0, \tau_{0}]}\lozenge_{[\tau_{1},\tau_{2}]}^{w}\mu_{i}$}: In this case, the decoder maps the time interval $[\tau_{1},\tau_{2}]$ to a vector $W^{2} \in \mathbb{R}^{n\times  1}$, and the $j-th$ element can be defined as:
\begin{equation}
W^{2}_{j} = D(\theta_{d},[\tau_{1},\tau_{2}])_{j}
\end{equation}
Then the activation function   can be defined as:
\begin{equation}
g^{\square\lozenge}(\rho^{2}_{i},W^{2})  = \rho^{w}(x,\wedge_{j=0:\tau_{0}}^{\textbf{I}}\lozenge_{[\tau_{1}+j,\tau_{2}+j]}^{W^{2}}\mu_{i},t),
\end{equation}
where $ \rho^{w}(x,\wedge_{j=0:\tau_{0}}^{\textbf{I}}\lozenge_{[\tau_{1}+j,\tau_{2}+j]}^{W_{i}^{2}}\mu_{i},t)$ is recursively defined in (\ref{eqn:weightedrobustness}) and $\textbf{I}\in \mathbb{R}^{\tau_{0}+1}$ is a vector whose entities are 1.

\textbf{(4)} \underline{$\lozenge_{[0,\tau_{0}]}\square_{[\tau_{1},\tau_{2}]}^{w}\mu_{i}$}: In this case, the decoder also maps the time interval to a matrix $W_{i}$ as the case in  \underline{$\square_{[0, \tau_{0}]}\lozenge_{[\tau_{1},\tau_{2}]}^{w}\mu_{i}$}, and the activation function can be defined as:
\begin{equation}
g^{\lozenge\square}(\mu_{i},W_{i})  = \rho^{w}(x,\vee_{j=0:\tau_{0}}^{\textbf{I}}\square_{[\tau_{1}+j,\tau_{2}+j]}^{W_{i}^{j}}\mu_{i},t),
\end{equation}
where $\rho^{w}(x,\vee_{j=0:\tau_{0}}^{\textbf{I}}\square_{[\tau_{1}+j,\tau_{2}+j]}^{W_{i}^{j}}\mu_{i},t)$ is recursively defined in (\ref{eqn:weightedrobustness}) and $\textbf{I}\in \mathbb{R}^{\tau_{0}+1}$ is a vector whose entities are 1.

\textbf{Layer 4} (Formula reduction layer): This layer includes an ``and'' neuron and an ``or'' neuron. The ``and'' neuron performs a logic ``and'' operation of its inputs, and the ``or'' neuron performs a logic ``or'' operation of its inputs, respectively. Denote the output from Layer 3 is a vector $\rho^{3} =[\rho^{3}_{i},\cdots,\rho^{3}_{M}]$, the activation function for the ``and'' neuron is defined as:
\begin{equation*}
g^{\wedge}(\rho^{3}, W^{3}) = \left\lbrace \begin{array}{ll}
\sqrt[M]{\Pi_{i=1,\cdots,M}(1+W^{3}_{i}\rho^{3}_{i})}-1,  \forall i, o_{i}>0\\
\frac{1}{M}\sum_{i=1,\cdots,M}[W^{3}_{i}\rho^{3}_{i}]_{-},\quad \text{ otherwise}
\end{array}\right.
\end{equation*}
where $W^{3}$ is the parameter  for the connections between Layer 3 and Layer 4, and $W^{3}_{i}$ is the $i-th$ element of $W^{3}$. Then the activation function for  ``or'' neuron can be defined accordingly as:
\begin{equation*}
g^{\vee}(\rho^{3}, W^{3}) = \left\lbrace \begin{array}{ll}
\frac{1}{M}\sum_{i=1,\cdots,M}[W^{3}_{i}\rho^{3}_{i}]_{+},  \exists i\in [1, N], \rho^{3}_{i}>0\\
-\sqrt[M]{\Pi_{i=1,\cdots,M}(1-W^{3}_{i}\rho^{3}_{i})}+1,  \text{ otherwise}
\end{array}\right. 
\end{equation*}

\textbf{Layer 5} (Output layer): The output layer has only one ``and''	 neuron and performs a logic ``and'' of the two outputs of layer 4. The activation function of this layer can be defined as:
\begin{equation*}
g^{\wedge}(\rho^{4}, W^{4}) = \left\lbrace \begin{array}{ll}
\sqrt[2]{\Pi_{i=1,2}(1+W^{4}_{i}\rho^{4}_{i})}-1,  \forall i, \rho^{4}_{i}>0\\
\frac{1}{2}\sum_{i=1,2}[W^{4}_{i}\rho^{4}_{i}]_{-},\quad \text{ otherwise}
\end{array}\right.
\end{equation*}
where $W^{4}$ is the parameters for the connections between Layer 4 and Layer 5.
With the definitions of the activation function, we have the following theorem.

\begin{strip} 
\rule{\textwidth}{0.4pt}
\begin{align}
\begin{array}{lll}
\rho^{w}(x, \wedge_{i=1:N}^{w}\varphi_{i}, t)&: = \left\lbrace  
\begin{array}{ll}
\sqrt[N]{\Pi_{i=1,\cdots,N}(1+w_{i}\rho^{w}(x,\varphi_{i},t)}-1,\quad \forall i\in [1,\cdot, N], \rho^{w}(x,\varphi_{i},t)>0\\
\frac{1}{N}\sum_{i=1,\cdots,N}[w_{i}\rho^{w}(x,\varphi_i,t)]_{-},\quad \text{ otherwise}
\end{array}\right. \\
\rho^{w}(x, \vee_{i=1:N}^{w}\varphi_{i},t)&:= \left\lbrace \begin{array}{ll}
\frac{1}{N}\sum_{i=1,\cdots,N}[w_{i}\rho^{w}(x,\varphi_i,t)]_{+},\quad \exists i\in [1,  N], \rho^{w}(x,\varphi_{i},t)>0\\
-\sqrt[N]{\Pi_{i=1,\cdots,N}(1-w_{i}\rho^{w}(x,\varphi_{i},t)}+1,\quad \text{ otherwise}
\end{array}\right. \\
\rho^{w}(x, \square_{[\tau_{1},\tau_{2}]}^{w}\varphi_{i}, t)&:= \left\lbrace \begin{array}{ll}
\sqrt[\tau_{2}-\tau_{1}]{\Pi_{\tau_{1}}^{\tau_{2}}(1+w_{i}\rho^{w}(x,\varphi_{i},t)}-1,\quad \forall \tau\in [t+\tau_{1},t+\tau_{2}], \rho^{w}(x,\varphi_{i},t)>0\\
\frac{1}{\tau_{2}-\tau_{1}}\sum_{t'\in[\tau_{1}+t,\tau_{2}+t]}[w_{i}\rho^{w}(x,\varphi,t')]_{-},\quad \text{ otherwise}
\end{array}\right. \\
\rho^{w}(x, \lozenge_{[\tau_{1},\tau_{2}]}^{w}\varphi, t)&:=\left\lbrace \begin{array}{ll}
\frac{1}{\tau_{2}-\tau_{1}}\sum_{t'\in [t+\tau_{1},t+\tau_{2}]}[w_{i}\rho^{w}(x,\varphi,t')]_{+},\quad\exists t'\in [t+\tau_{1},t+\tau_{2}], \rho^{w}(x,\varphi_{i},t')>0\\
-\sqrt[\tau_{2}-\tau_{1}]{\Pi_{\tau_{1}}^{\tau_{2}}(1-w_{i}\rho^{w}(x,\varphi,t)}+1, \quad \text{ otherwise}
\end{array}\right. 
\end{array}
\label{eqn:weightedrobustness}
\end{align}

\end{strip}

\begin{theorem}
(Soundness): Denote the embedded wSTL formula of the TLNN to be $\varphi$ and the output of the network  to be $TLNN(\Theta,x)$, where $\Theta$ is the set of parameters for the network, then we have 
\begin{equation}
\begin{array}{ll}
TLNN(\Theta, x)\geq  0 \Leftrightarrow\rho^{w}(x,\varphi,0)\geq 0 \Rightarrow x\models\varphi,\\
TLNN(\Theta, x)< 0 \Leftrightarrow\rho^{w}(x,\varphi,0)<0 \Rightarrow x\nvDash\varphi.\\
\end{array}
\end{equation}
\end{theorem}
\begin{proof}
(sketch) Note that the weights are  not directly applied to the signals, but combined with the robustness of predicates or atomic formulas, thus the weights soundness property of traditional STL. Thus the soundness property is preserved.
\end{proof}
\section{TLNN Learning}
\label{sec:learning}
This section introduces the learning method of TLNN. Initially, there exists  only one neuron in layer 3. More neurons can be added based on simultaneous structure and parameter learning. The python code of TLNN can be found at \url{https://github.com/datagangchen/TLNN}.
\subsection{Structure Learning}

The online structure learning is based on the importances of the neurons. The output of layer 5 and weights of layer 4 are used to determine whether a new neuron should be generated. The structure-learning algorithm for the TLNN can be described in two steps: neuron reduction and  increase.

\textbf{Neuron reduction}: Neuron reduction is performed in layer 4, where we define a threshold $w_{th}$ for the weights between layer 3 and layer 4. Denote the weight vector is $W^{3}$ and $W^{3,j}_{i}$ denotes the weight between the $i-th$ neuron in layer 3 to the  $j-th$ neuron in layer 4, where $i\in 1,\cdots, M$ and $j = 1,2$. Then the neuron reduction step can be described as follows:
\begin{equation}
W^{3,j}_{i} = \left\lbrace 
\begin{array}{ll}
0, \text{ if } W^{3,j}_{i}<w_{th}\\
W^{3,j}_{i},\text{ if } W^{3,j}_{i}\geq w_{th}
\end{array}\right. 
\end{equation}

\textbf{Neuron increase}: It is possible that the neural network cannot approach a good result with the current neurons. In this case, we need to add more neurons to the network to increase its expressiveness. Here we define a threshold  for the cost function defined over the output result of layer 5. The cost function is defined as:
\begin{equation}
C = \frac{1}{|\mathcal{D}|}\sum_{(x,y, y_{d})\in \mathcal{D}}(y-y_{d})^{2}
\end{equation}
where $\mathcal{D}$ represents the data set, which has the form $(x, y, y_{d})$. $x$ denotes the signal, $y$ denotes the output from the TLNN at layer 5, and $y_{d}$ denotes the desired output, respectively. Given a threshold $c_{th}$, when $C>c_{th}$, a new neuron will be added to layer 3. The type of which is randomly chosen from four kinds of atomic formula. 

\subsection{Parameter Learning}
The parameter learning step is performed simultaneously to the structure-learning step. The gradient descent algorithm is used to train the parameters of the network at each time when an incoming sample is given, which is suitable for a supervised method. The aim of the learning process is to minimize the error function
\begin{equation}
L=\frac{1}{2}(y-y_{d})^{2}
\label{eqn:loss}
\end{equation}
Next, we describe the antecedent parameter learning for TLNN based on the gradient descent algorithm as follows.
Let $\theta$ denote the antecedent parameter that need to be trained, then $\theta$ should be updated at $k-th$ step as 
\begin{equation}
\theta (k+1) = \theta (k)-\eta\frac{\partial L}{\partial \theta (k)}
\end{equation}
where $\eta$ is a learning coefficient in the range $[0,1]$ and $\theta$ can be any parameter of the network. Next, we will derive the gradient $\frac{\partial L}{\partial \theta}$ step by step.

\textbf{Gradient for output layer}: There are two weight parameters specifying the weights    for the output of ``and'' neuron and ``or'' neuron in Layer 4, denoted as $W^{4}_{1}$ and $W^{4}_{2}$, respectively. The gradient of the parameters are defined as follows.
\begin{equation}
\frac{\partial L}{\partial W^{4}_{i}} =\frac{\partial L}{\partial \rho^5} \frac{\partial \rho^{w}(x,\wedge^{W^{4}}_{i=1,2}\rho^{4}_{i},t)}{\partial W^{4}_{i}}
\end{equation}
 where $\rho^{5}$ is the output of the network, $\rho^{4}_i$ is the output of the previous layer to neuron $i$ and the partial differentiation $\frac{\partial \rho^{w}(x,\wedge^{W^{4}}_{i=1,2}\rho^{4}_{i},t)}{\partial W^{4}_{i}}$ can be obtained with the robustness equation in (\ref{eqn:weightedrobustness}). The loss function is defined in (\ref{eqn:loss}), thus $\frac{\partial L}{\partial \rho^5} = (y-y_{d})$.

\textbf{Gradient for formula reduction layer}: There are two kinds of neurons in this layer, thus the gradient of the two neurons are different. Denote $W^{3,a}_{i}$ is the parameter of the $i-th$ input weight for ``and'' neuron, and $W^{3,o}_{i}$ is defined accordingly for the ``or'' neuron. The gradients for the parameters  of ``and'' neuron are
\begin{equation}
\frac{\partial L}{\partial W^{3,a}_{i}} =\frac{\partial L}{\partial \rho^5} \frac{\partial \rho^{5}}{\partial \rho_{a}^{4}}\frac{\partial \rho_{a}^{4}}{\partial W^{3,a}_{i}}
\end{equation}
where $\rho_{a}^{4}$ is the output of the ``and''  neuron in layer 4, and $\frac{\partial \rho^{5}}{\partial \rho^{4}_{a}}$ is the differentiation of the output in layer 5 with respect to the output of the ``and'' neuron in layer 4. The differentiation $\frac{\partial \rho_{a}^{4}}{\partial W^{3,a}_{i}}$ can be obtained by $\frac{\partial \rho^{w}(x,\wedge^{W^{3}}_{i=1:M}\rho^{3}_{i},t)}{\partial W^{3,a}_{i}}$ derived with the robustness defined in (\ref{eqn:weightedrobustness}).  Then the    gradients for the parameters  of ``or'' neuron is
\begin{equation}
\frac{\partial L}{\partial W^{3,o}_{i}} =\frac{\partial L}{\partial \rho^5} \frac{\partial \rho^{5}}{\partial \rho_{o}^{4}}\frac{\partial \rho_{o}^{4}}{\partial W^{3,o}_{i}}
\end{equation}
where $\frac{\partial \rho_{a}^{4}}{\partial W^{3,o}_{i}}$ can be defined by $\frac{\partial \rho^{w}(x,\vee^{W^{3}}_{i=1:M}\rho^{3}_{i},t)}{\partial W^{3,o}_{i}}$  accordingly.

\textbf{Gradient for atomic formula layer}: In this layer, there are four kinds of neurons, and each of the gradient has a different form. In this layer, the outputs of layer 2 are directly passed to layer 3 without assigning weights, thus the parameters that need to be updated are the parameters of the autoencoder.  We denote the $i-th$ parameter of $j-th$ neuron   as $\theta^{j}_{i}$, where $j=1,\cdots, M$ and $l=1,\cdots,4$ and the number of $i$ depends on the structure of the autoencoder. Then the gradient for $\theta^{j}_{i}$ can be defined as 
\begin{equation}
\frac{\partial L}{\partial \theta^{j}_{i}}=\frac{\partial L}{\partial \rho^{5}}(\frac{\partial  \rho^{5}}{\partial \rho^{4}_{a}}\frac{\partial \rho^{4}_{a}}{\partial \rho^{3}_{j}}+\frac{\partial \rho^{5}}{\partial \rho^{4}_{o}}\frac{\partial \rho^{4}_{o}}{\partial \rho^{3}_{j}})\frac{\partial \rho^{3}_{j}}{\partial \theta^{j}_{i}}
\end{equation}
where $\rho^{3}_{j}$ is the output of the $j-th$ neuron, $\frac{\partial L}{\partial \rho^{5}}$, $\frac{\partial  \rho^{5}}{\partial \rho^{4}_{a}}$, $\frac{\partial \rho^{4}_{a}}{\partial \rho^{3}_{j}}$, $\frac{\partial \rho^{5}}{\partial \rho^{4}_{o}}$, and $\frac{\partial \rho^{4}_{o}}{\partial \rho^{3}_{j}}$ can be obtained based on the robustness definition in (\ref{eqn:weightedrobustness}).  $\frac{\partial \rho^{3}_{j}}{\partial \theta^{j}_{i}}$ is the standard differentiation for autoencoder, which can be obtained easily by follow the method described in \cite{nguyen2019dynamics}. Note that $\rho^{3}_{j}$ has different definitions for different type of atomic formulas.

\textbf{Gradient for predicate layer}: There are $M\times n$ neurons in this layer, but there are only $M$ weight parameters given, since each neuron in layer 3 has only one predicate in the form $f(x)-W^{1}_j\geq 0$ or $f(x)-W^{1}_j<0$, and $W^{1}_j$ is the parameter that needs to be updated.  Then the gradient for $W^{1}_j$ can be defined as 
\begin{equation*}
\frac{\partial L}{\partial W^{1}_j}=\frac{1}{n}\frac{\partial L}{\partial \rho^{5}}\sum_{i=1:n}
(\frac{\partial \rho^{5}}{\partial \rho^{4}_{a}}\frac{\partial \rho^{4}_{a}}{\partial \rho^{3}_{j}}+\frac{\partial \rho^{5}}{\partial \rho^{4}_{o}}\frac{\partial \rho^{4}_{o}}{\partial \rho^{3}_{j}})\frac{\partial \rho^{3}_{j}}{\partial \rho^{2,j}_{i}}\frac{\partial \rho^{2,j}_{i}}{\partial W^{1}_{j}}
\end{equation*}
where $\rho^{2,j}_{i}$ is the output to $j-th$ neuron of Layer 2 in the $i$ input group, which is computed with the $i$ neuron of Layer 1.  $\frac{\partial \rho^{3}_{j}}{\partial \rho^{2,j}_{i}}$ can be defined according to the type of neuron in Layer 3, which can follow the gradient in the atomic formula layer. $\frac{\partial \rho^{2,j}_{i}}{\partial W^{1}_{j}}$ is $1$ or $-1$ one according to the type of predicates.
\begin{figure}[t]
\centering
\includegraphics[scale=0.28]{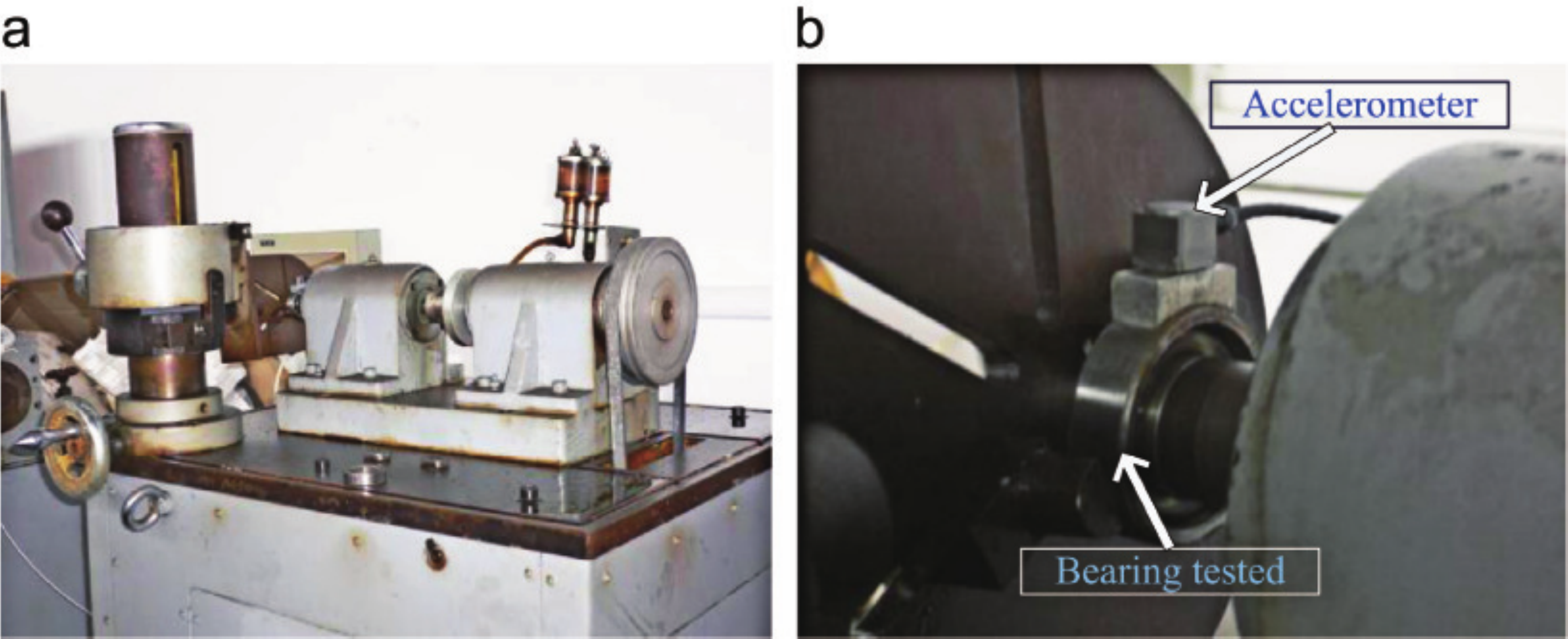}
\caption{(a) The   test rig and (b) the location of the accelerometer for signal collection. }
\label{fig: testrig}
\end{figure}
\section{Case Studies}
\label{sub:case}

In this section,  we will evaluate the performance of our TLNN for fault diagnosis and interpretation with experimental data over a set of rolling element bearings.   The rotational machine under study is a rolling-element bearing test-rig,  as shown in Figure \ref{fig: testrig}, which has been used in \cite{jiang2015study}. To collect the faulty signals of rolling element bearings, the electrical-discharge machining method was used to introduce single pitting faults. The faults are added to the surface of the race or the rolling body of a series of rolling element bearings (one type of fault for each). The signals are collected with the shaft speed being fixed and the sampling rate is 12 kHz.

During the experiment,  we introduce three kinds of faults to the bearings, i.e., rolling element fault, inner race fault, and outer race fault. Then we collect data for four conditions of the bearings, i.e., rolling element fault, inner race fault, outer race fault, and normal bearing. After all the data has been collected, 220  signal samples with length 1024 for each condition were used for demonstration (880 pieces in all). To obtain the input signals for TLNN, we first decompose the signals with wavelet package transform at level two, then we calculate the second temporal moment with the Matlab embedded function for each piece of the signals. With the second temporal moment for each decomposed signal, we connect the decomposed signals into one and sample the obtained signals to get a shorter signal, which has a length of 128. Therefore, we have 220  samples for each bearing condition.   Then we construct the labelled training and testing set. The positive training set for inner fault includes 110 samples from inner fault signals. The labels of which are $1$, and the negative set includes 90 samples from the other three conditions (30 samples for each). The labels of which are $-1$. Then the training set for the other two conditions is constructed accordingly. The positive testing set for inner fault includes the rest of 110 samples from inner fault signals, and the negative testing set includes 90 samples from the other three conditions (30 samples for each). We construct the training and testing sets for each fault independently and accordingly.

\begin{figure}[hbtp]
\centering
\includegraphics[scale=0.64]{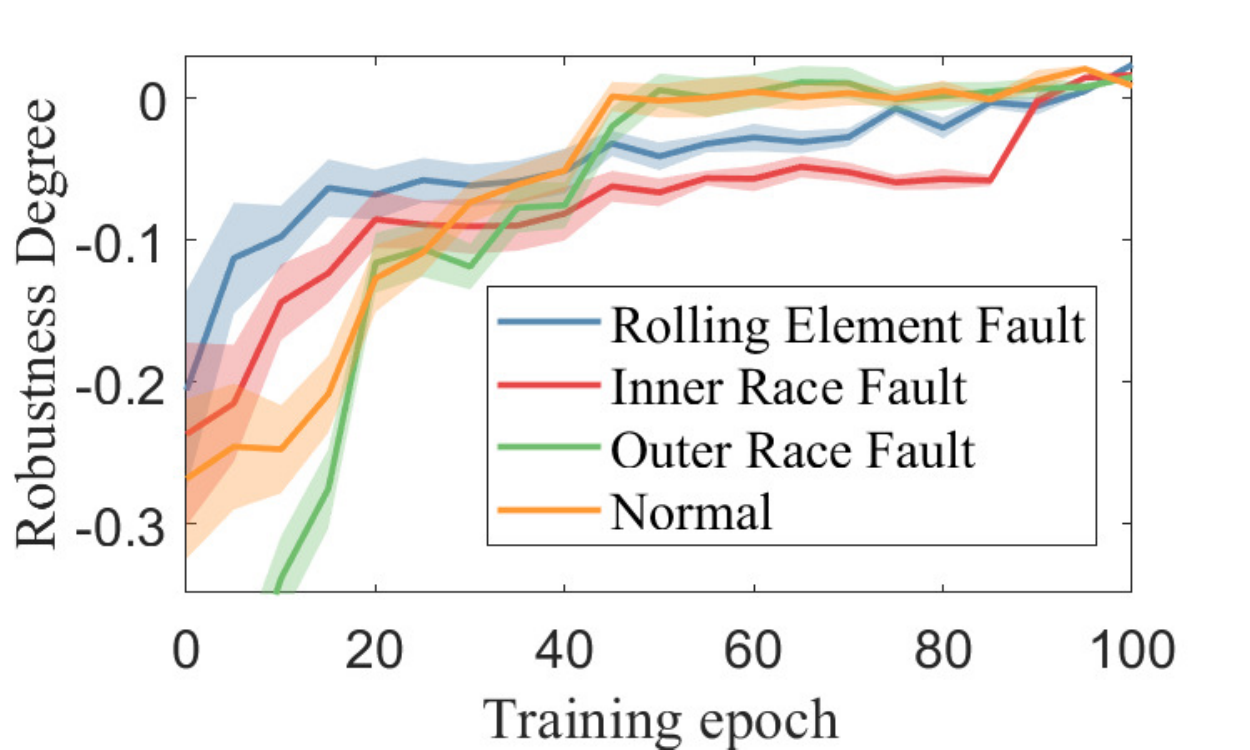}
\caption{Average robustness degree and its variance obtained for each training epoch among testing data sets with the TLNN.}
\label{fig:misclass}
\end{figure}
%

 \begin{table*}[t]
\centering
\caption{ sTFPG associated STL formulas   for inner race, outer race and rolling element faults with real experiment data.}
 \resizebox{\textwidth}{!}{
    \begin{tabular}{llllll} 
    \hline \hline\\[-2.5mm]
 \textbf{Fault Type} &{Interpretation}&\multicolumn{2}{c}{\textbf{Robustness Degree}} &\multicolumn{2}{c}{\textbf{Error Rate}} \\[0.5ex]
      \hline
      -& -& Training & Testing&  Training & Testing\\
      Inner Race &  $\varphi_{I} = \lozenge_{[0,5]}\square_{[58,66]}(x\geq 0.05)\wedge \square_{[14,31]}(x<0.3)\wedge \square_{[45,52]}(x<0.04)$ &0.0012 &0.0082& 0.000&0.000   \\
      \hline
      Outer Race &$\varphi_{O}=\lozenge_{[60,68]}(x<0.1)\wedge \lozenge_{[44,50]}(x\geq 0.08)\wedge \lozenge_{[0,5]}\square_{[18,30]}(x< 0.3) $& 0.0022 & -0.030&  0.000  &0.045 \\
      \hline
      Rolling Element& $\varphi_{R}=\lozenge_{[0,5]}\square_{[20,25]}(x<0.1)\wedge \square_{[65,72]}(x\geq 0.3)$ & 0.0022 & -0.030&  0.000  &0.045  \\
      \hline
      Normal& $\varphi_{N}= \lozenge_{[0, 5]}\square_{[15,25]}(x\geq 0.4)\vee \square_{[0, 5]}(x\geq 0.12)$ &0.043 &0.011& 0.000 &0.000 \\      
      \hline \hline
    \end{tabular}
}
      \label{tab:table4}
\end{table*}

%

Table \ref{tab:table4} shows  the learning results of the TLNN for the four conditions, where the weights are omitted since they do not help interpret the faults.  The formulas can be seen as the interpretation of the faults. For example, for the inner race fault, the formula $\varphi_{I}$ can be read ``eventually within $[0,5]$, always within $[58,66]$, the signals should be larger  than or equal to 0.05, and always within $[14,31]$, the signals should be smaller than 0.3, and always within $[45,52]$, the signals should be smaller than 0.04.''  The other faults can be interpreted in the same way. The visualization of these formulas can be found in Fig. \ref{fig:visuall_inner}-\ref{fig:visuall_normal}, where the blue signals are positively labelled, and red signals are negatively labelled, and the blue signals should avoid the yellow regions but reach the green regions. Note that the trajectories are not time series, but spectral of the signals. Here  we can consider the frequency as time and use index values to map the signals to   values, such that the TLNN can be applied. Moreover, Fig. \ref{fig:visuall_inner}-\ref{fig:visuall_normal} only show a portion of the data set for better visuals.

\begin{figure}[hbtp]
\centering
\includegraphics[scale=0.65]{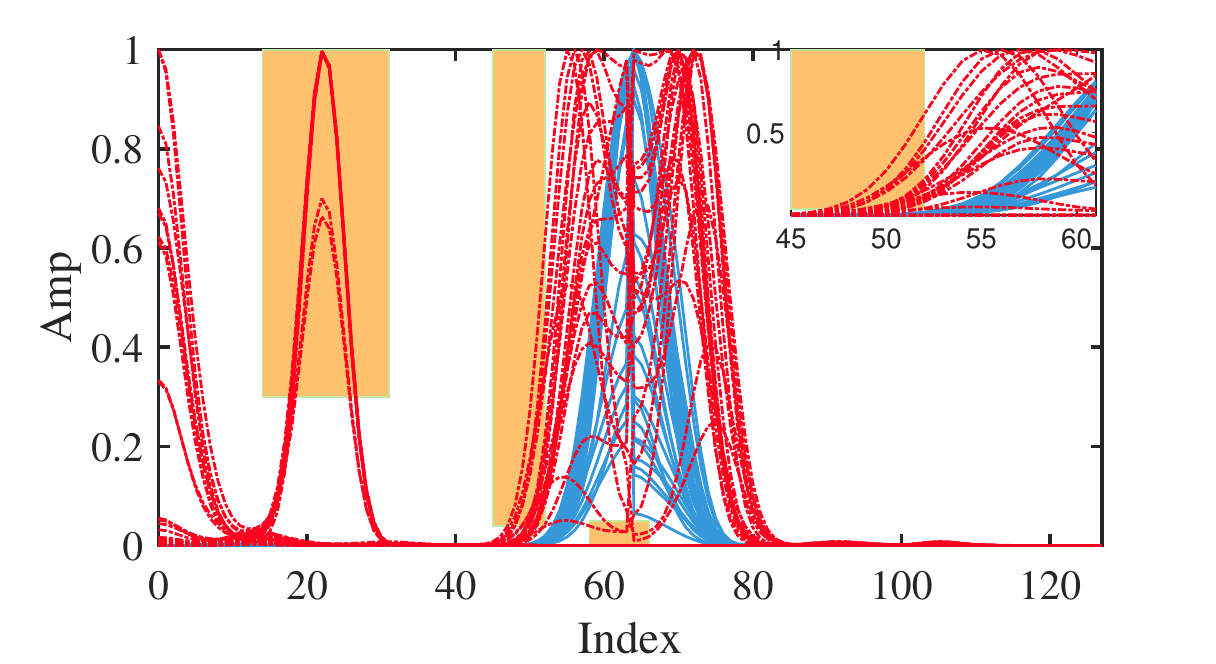}
\caption{Formal interpretation visualization for $\varphi_{I}$.}
\label{fig:visuall_inner}
\end{figure}

\begin{figure}[hbtp]
\centering
\includegraphics[scale=0.65]{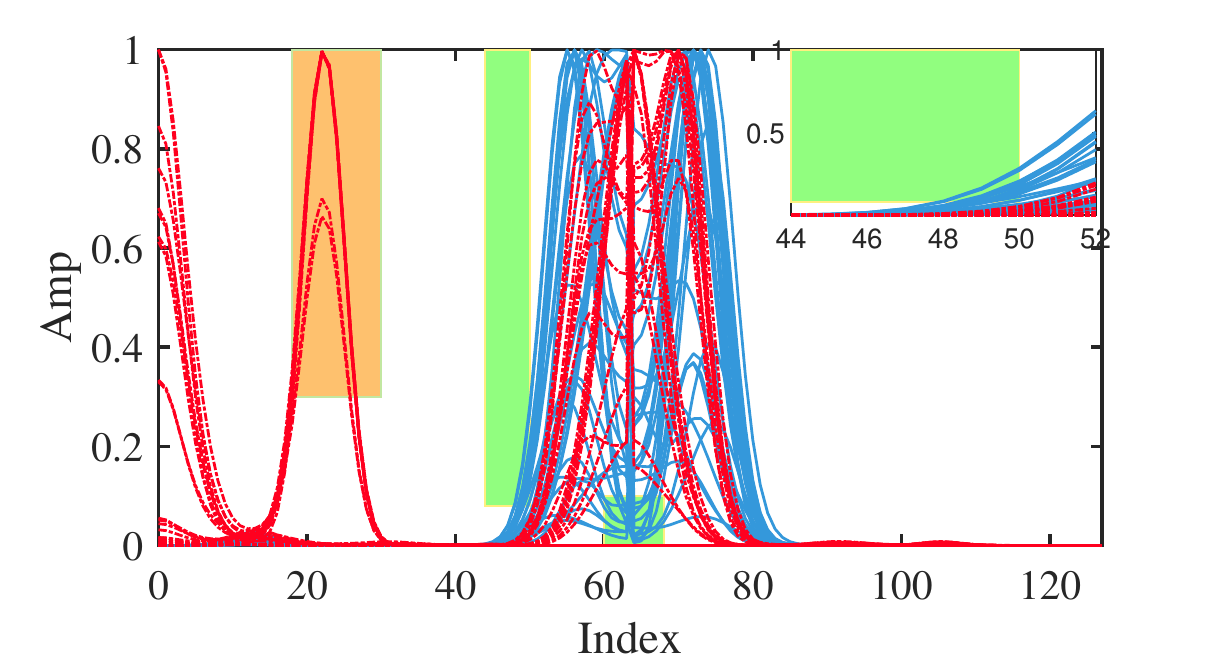}
\caption{Formal interpretation visualization for $\varphi_{O}$.}
\label{fig:visuall_outer}
\end{figure}

\begin{figure}[hbtp]
\centering
\includegraphics[scale=0.65]{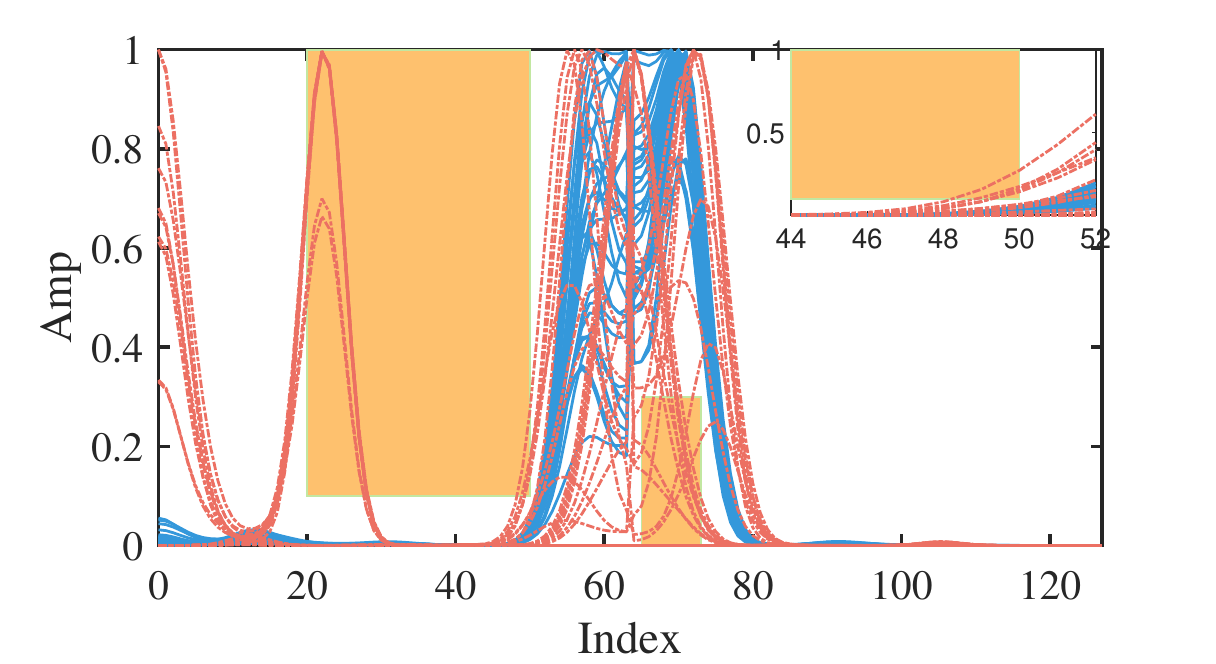}
\caption{Formal interpretation visualization for $\varphi_{R}$.}
\label{fig:visuall_ball}
\end{figure}

\begin{figure}[hbtp]
\centering
\includegraphics[scale=0.65]{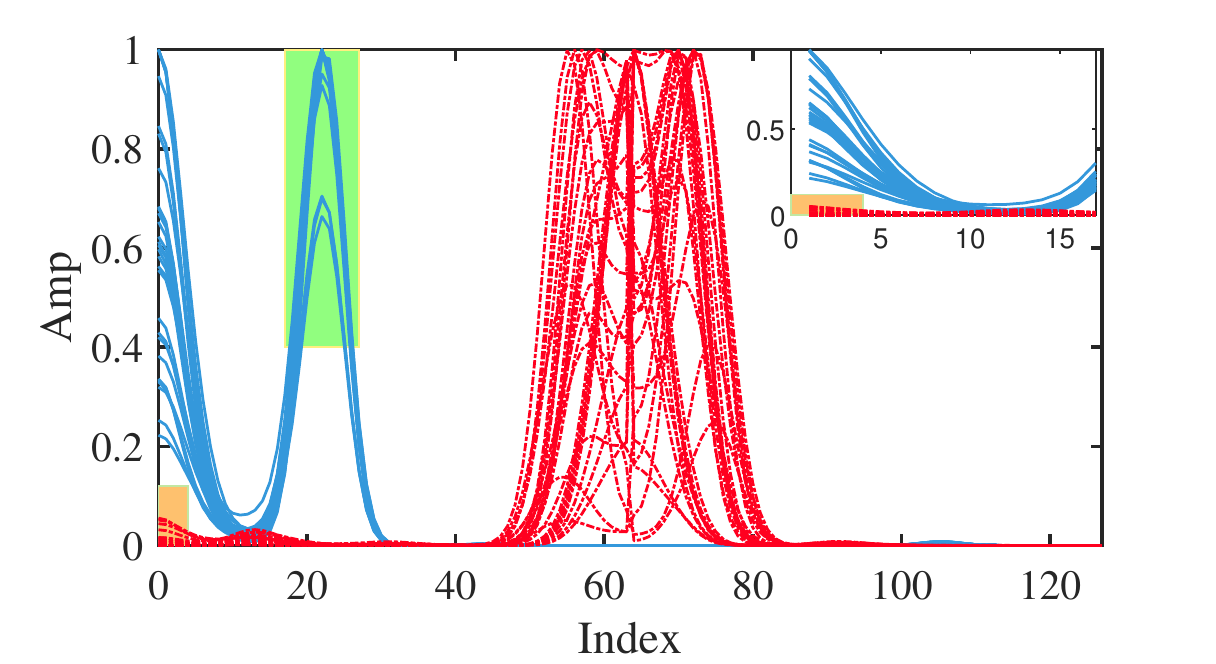}
\caption{Formal interpretation visualization for $\varphi_{N}$.}
\label{fig:visuall_normal}
\end{figure}
The learned wSTL formulas have some components denoting the signals that are larger than some values within some time intervals, and showing that the occurring of fault conditions will lead to a larger value for the energy at the time within these intervals.  Moreover,  different faults will have different concentration patterns in the frequency domain. Namely, the wSTL formulas denote  a sequence of impulse energies. Since the fault mechanism of rolling-element bearing is that the impulse energy comes from the strikes of rollers on the fault surface and excites the striking response of the bearing system, which means our wSTL formula reveals the same knowledge of the fault mechanism. With this knowledge, human maintainers know that avoiding energy concentration by adding lubrication will reduce the risk of system failure.

Table \ref{tab:table4} also shows the average faults diagnosis results for the experimental data with the learned wSTL formulas among 10 trails, in which the robustness and error rate are used as the metrics. Note that a formula with a positive robustness indicates it diagnoses the fault correctly.  On the contrary, a formula with negative robustness may lead to misdiagnosis. For example, the outer race has a negative robustness among the testing data sets, therefore leading to the misdiagnosis in Table \ref{tab:table4}, which may be caused by the noise or the differences of fault patterns or distribution between training and testing data.   Our results not only show that our method can find the wSTL formulas that are in line with the failure mechanisms of rolling element bearing, but also achieve a fault diagnosis error that is less than 5\% with the test data set.  Moreover, the increase of robustness may not lead to a decrease in error rate, since we can only guarantee that positive robustness leads to a zero error rate.

In the second experiment, we further investigate the properties of our method by conducting a comparison experiment with the state-of-the-art formal logic-based methods, in which we compare the performance between our method and the temporal logic-based method in \cite{kong2016temporal} and frequency temporal logic method (FTL) in \cite{chen2020frequency}  over the rolling element fault case. In this experiment,     we check the performance of these methods at 20, 40, 60, and 80 minutes during the training process.   The results are shown in Table \ref{tab:table5}, which shows that our method and the method in \cite{kong2016temporal} can achieve zero miss-classification rate within 40 minutes, while the method in \cite{chen2020frequency} cannot fully diagnose the faults correctly. \cite{kong2016temporal} can obtain a good performance since it searches along with a predefined order for the optimal formulas. When the length of the formula is small, it can obtain a good performance. The method in \cite{chen2020frequency} tries to use a Gaussian Process to approximate the robustness degree function, which is a hard task due to the non-convex  and non-smooth properties of the robustness degree function. Therefore, \cite{chen2020frequency} cannot reach a good performance within a limited time. Moreover, the robustness degree obtained with the proposed method is usually larger than the other two, since the TLNN tried to approximate the target value, which is 1, while the other methods tried to find a positive robustness. In the other word, the robustness obtained in this paper does not reveal how much the signals satisfy the wSTL formula.   However, we can ignore the weights in wSTL, which leads to an STL formula, and calculate the robustness for the obtained STL formula to find how much the signals satisfy the formula.

\begin{table}[!ht]
\begin{center}
   \caption{Comparison results between the proposed method and the other logic based methods for rolling element fault}\label{tab:table5}
   \resizebox{\columnwidth}{!}{
    \begin{tabular}{lcccc} 
    \hline \hline\\[-3mm]
   \textbf{Method} &  \multicolumn{4}{c}{Error Rate/ Robustness}\\
     \hline 
Time (min)    & 20  & 40 & 60 & 80\\
      \hline
      Proposed Method                         &0.05/-0.0034    &\textbf{0.00/0.243} &\textbf{0.00/0.412}& \textbf{0.00/0.552}\\
      FTL \cite{chen2020frequency} & \textbf{0.31/-0.312} &0.135/-0.213  &0.050/-0.0016 &0.015/-0.0013\\    
      Temporal Logic \cite{kong2016temporal}  &0.320/-0.343     &0.250/-0.115   &0.010/-0.012 & 0.00/0.023\\
      [.1ex]\hline\hline
    \end{tabular}
  }
\end{center}
\end{table}

\section{Conclusions}
\label{sub:conclusion}
The paper presents a novel neural network, which is called a temporal logic neural network. The proposed neural network keeps the nice properties of traditional neural networks, but can be interpreted with a weighted signal temporal logic.    The advantages of the proposed method have been demonstrated with experiments with real data sets. The experiment results indicate the proposed neural network can achieve high performance for fault diagnosis in terms of accuracy and computation efficiency. Moreover, the embedded formal language of the neural network can provide an explanation of the decision process.  Given the popularity of formal language in safety-critical systems, we believe this paper provides a necessary foundation for many future system control monitoring frameworks.

\bibliography{references}
\bibliographystyle{IEEEtran}

%
%
%
%

\end{document}